%% file: aaai-main.tex
\documentclass[letterpaper]{article} 
\usepackage{aaai2026}  
\usepackage{times}  
\usepackage{helvet}  
\usepackage{courier}  
\usepackage[hyphens]{url}  
\usepackage{graphicx} 
\usepackage{natbib}  
\usepackage{caption} 
\usepackage{algorithm}
\usepackage{algorithmic}

\usepackage{newfloat}
\usepackage{listings}
\DeclareCaptionStyle{ruled}{labelfont=normalfont,labelsep=colon,strut=off} 
\floatstyle{ruled}
\newfloat{listing}{tb}{lst}{}
\floatname{listing}{Listing}
\usepackage{CJKutf8}
\usepackage[utf8]{inputenc} 
\usepackage{amsmath, amssymb}
\usepackage{upgreek} 

\usepackage{subcaption} 

\usepackage{bm}

\usepackage{amsthm}
\newtheorem{theorem}{Theorem}[section]

\newtheorem{proposition}[theorem]{Proposition}
\newtheorem{corollary}[theorem]{Corollary}
\theoremstyle{definition}
\newtheorem{definition}[theorem]{Definition}

\theoremstyle{remark}
\newtheorem{remark}[theorem]{Remark}
\theoremstyle{definition}
\newtheorem{assumption}[theorem]{Assumption}

\renewenvironment{proof}{{\noindent\textit{Proof}.\ }}{\hfill $\square$ \par}

\usepackage{xspace}
\newcommand{\fft}{\sffamily{FFT}\rmfamily\xspace}
\newcommand{\lora}{\sffamily{LoRA}\rmfamily\xspace}
\newcommand{\freeze}{\sffamily{Freeze}\rmfamily\xspace}
\newcommand{\llamapro}{\sffamily{LLaMA~Pro}\rmfamily\xspace}
\newcommand{\widthpreserve}{\sffamily{Width(Preserve)}\rmfamily\xspace}
\newcommand{\widthadapt}{\sffamily{Width(Adapt)}\rmfamily\xspace}
\newcommand{\scale}{\sffamily{SCALE}\rmfamily\xspace}
\newcommand{\scalepreserve}{\sffamily{SCALE-Preserve}\rmfamily\xspace}
\newcommand{\scaleadapt}{\sffamily{SCALE-Adapt}\rmfamily\xspace}
\newcommand{\scaleroute}{\sffamily{SCALE-Route}\rmfamily\xspace}

\newcommand{\bW}{\bm{W}}
\newcommand{\bX}{\bm{X}}
\newcommand{\bZ}{\bm{Z}}

\newcommand{\g}{\nabla}

\usepackage{mathtools}      
\DeclarePairedDelimiter\p{\lparen}{\rparen}
\DeclarePairedDelimiter\ang{\langle}{\rangle} 
\DeclarePairedDelimiter\abs{\lvert}{\rvert}   
\DeclarePairedDelimiter\norm{\lVert}{\rVert}  
\DeclarePairedDelimiter\bkt{[}{]}             
\DeclarePairedDelimiter\set{\{}{\}}           
\DeclarePairedDelimiter\ceil{\lceil}{\rceil}
\DeclarePairedDelimiter\floor{\lfloor}{\rfloor}

\makeatletter
\let\oldp\p \def\p{\@ifstar{\oldp}{\oldp*}}
\let\oldang\ang \def\ang{\@ifstar{\oldang}{\oldang*}}
\let\oldabs\abs \def\abs{\@ifstar{\oldabs}{\oldabs*}}
\let\oldnorm\norm \def\norm{\@ifstar{\oldnorm}{\oldnorm*}}
\let\oldbkt\bkt \def\bkt{\@ifstar{\oldbkt}{\oldbkt*}}
\let\oldset\set \def\set{\@ifstar{\oldset}{\oldset*}}
\let\oldceil\ceil \def\ceil{\@ifstar{\oldceil}{\oldceil*}}
\let\oldfloor\floor \def\floor{\@ifstar{\oldfloor}{\oldfloor*}}
\makeatother

\newcommand{\normal}[1]{\normalfont{\text{#1}}}

\newcommand{\w}{\normalfont{\textbf{w}}}
\newcommand{\x}{\normalfont{\textbf{x}}}

\def\gE{\mathbb{E}}
\def\gI{\mathbb{I}}
\def\gP{\mathbb{P}}
\def\gD{\mathcal{D}}

\DeclareMathOperator*{\argmin}{arg\,min}

\usepackage[dvipsnames]{xcolor}
\newcommand{\sukhoon}[1]{}
\newcommand{\jinwoo}[1]{}

\title{\scale: Upscaled Continual Learning of Large Language Models}
\author{
    Jin-woo Lee\equalcontrib,
    Junhwa Choi\equalcontrib,
    Bongkyu Hwang,
    Jinho Choo,
    Bogun Kim,
    JeongSeon Yi,
    Joonseok~Lee,
    DongYoung Jung,
    Jaeseon Park,
    Kyoungwon Park,
    Suk-hoon Jung\thanks{Corresponding author: sukhoon.jung@samsung.com}
}
\affiliations{
    Samsung SDS

}

\usepackage{bibentry}

\begin{document}
\begin{CJK}{UTF8}{mj}

\maketitle

\begin{abstract}
We revisit continual pre-training for large language models and argue that progress now depends more on scaling the right structure than on scaling parameters alone. We introduce \scale, a width upscaling architecture that inserts lightweight expansion into linear modules while freezing all pre-trained parameters. This preserves the residual and attention topologies and increases capacity without perturbing the base model’s original functionality. \scale is guided by two principles: \emph{Persistent Preservation}, which maintains the base model’s behavior via preservation-oriented initialization and freezing of the pre-trained weights, and \emph{Collaborative Adaptation}, which selectively trains a subset of expansion components to acquire new knowledge with minimal interference. We instantiate these ideas as \scalepreserve (preservation-first), \scaleadapt (adaptation-first), and \scaleroute, an optional routing extension that performs token-level routing between preservation and adaptation heads. On a controlled synthetic biography benchmark, \scale mitigates the severe forgetting observed with depth expansion while still acquiring new knowledge. In continual pre-training on a Korean corpus, \scale variants achieve less forgetting on English evaluations and competitive gains on Korean benchmarks, with these variants offering the best overall stability-plasticity trade-off. Accompanying analysis clarifies when preservation provably holds and why the interplay between preservation and adaptation stabilizes optimization compared to standard continual learning setups.
\end{abstract}

\input{1_intro}
\input{2_related_work}
\input{3_upscaling_architecture}
\input{4_upscaled_learning}
\input{5_experiments}
\input{6_conclusion}

\bibliography{7_references}

\input{8_appendix}

\end{CJK}
\end{document}

%% file: 1_intro.tex
\section{Introduction}

\sukhoon{ continual learning이라는 단어 관련, 
.. CPT that preserves prior knowledge while acquiring new capabilities...
사전적인의미로는 continued learning과 다르게 continual learning이 신규지식을 배우는 것이 맞습니다만, 많은 논문에서 신규지식/기존지식증강 상관없이 중립적으로 쓰이기도 합니다. 
인트로에 domain adaptation/language extention 같은 목적성을 직접 기술하는 것은 어떨까요?
}


The era of effortless gains from brute-force scaling of large language models\,(LLMs) is nearing its end. Recent discussions among leading AI researchers emphasize that further progress will depend less on adding parameters or data and more on scaling \emph{the right structure} while preserving previously acquired knowledge\,\cite{openai,meta}. This view redirects attention to \emph{strategic architectural expansion}: approaches that enable continual pre-training\,(CPT) while preserving the pre-trained knowledge base. Classical continual learning\,(CL) methods—regularization, replay, or parameter isolation—help retain earlier knowledge but do not allocate new representational capacity for adaptation\,\cite{kirkpatrick2017overcoming,rolnick2019experience}. By contrast, function-preserving transformations such as Net2Net\,\cite{chen2015net2net} and its successors show that structural expansion can increase capacity without disrupting the base model's original functionality. However, depth upscaling such as \llamapro\,\cite{llamapro} often perturbs hidden representations and exhibits forgetting during CPT, indicating that architectural balance between preservation and adaptation must be managed more delicately.

\begin{figure}[t!]
\centering
    \includegraphics[width=.9\linewidth]{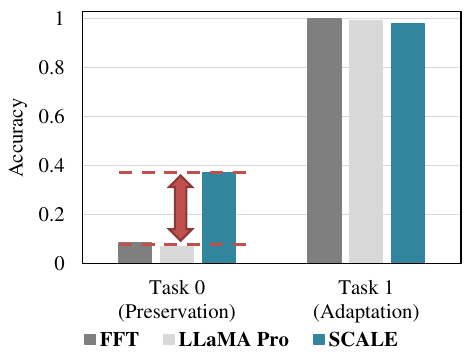}
    \caption{\textbf{Continual biography learning.}}
    \label{fig:biography}
\end{figure}
\begin{figure}[t!]
\centering
    \includegraphics[width=.9\linewidth]{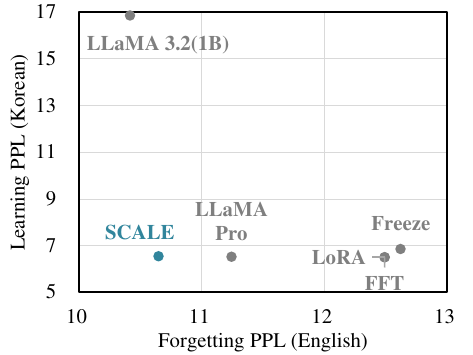}
    \caption{\textbf{Performance landscape.}}
    \label{fig:ppl}
\end{figure}

We therefore propose \scale\,(\textit{up\underline{S}caled \underline{C}ontinu\underline{A}l \underline{LE}arning}), a width upscaling architecture that introduces lightweight expansion inside linear modules while freezing all pre-trained parameters. \scale\ enlarges the representational space of decoder-style LLMs without altering the residual topology or attention structure. Empirically, Figure~\ref{fig:biography} shows that depth-upscaled \llamapro incurs severe forgetting on the \emph{continual biography} task\footnote{Refer to Section~\ref{sec:experiments} for details of continual biography task.}, whereas width-upscaled \scale\ preserves prior knowledge significantly better while adapting to new knowledge. In CPT on the Korean dataset, Figure~\ref{fig:ppl} highlights that \scale\ achieves lower English forgetting perplexity and competitive Korean learning perplexity compared with representative baselines, including \llamapro\ and \freeze\,\cite{zheng2025spurious}, reflecting a superior stability–plasticity balance.

\scale\ is instantiated by two complementary design principles developed empirically and supported theoretically: ① \emph{Persistent Preservation} and ② \emph{Collaborative Adaptation}. \emph{Persistent Preservation} maintains the original function throughout training via preservation-oriented initialization and freezing patterns within the expansion. \emph{Collaborative Adaptation} selectively trains a subset of expansion blocks\,(e.g., in upper layers or specific modules) to capture new domain knowledge while interacting stably with the frozen base. Our preliminary studies reveal that preservation-first configurations exhibit extreme resistance to forgetting, whereas collaborative configurations unlock strong adaptability; together they form a controllable frontier between stability and plasticity.

Based on the two principles, we introduce a family of width-upscaled learning methods:\,(i) \scalepreserve, which enforces preservation by maintaining the pre-trained function;\,(ii) \scaleadapt, which unlocks plasticity by collaborating with the frozen base; and\,(iii) \scaleroute, which routes tokens between preservation and adaptation paths using a similarity-based router. Because adaptation paths can override preservation paths, \scaleroute\ exposes both behaviors within a single forward pass and routes to the more relevant logits. We further derive a convergence advantage of routing-based CL over standard CL, theoretically supporting the observed stability–plasticity gains.

\paragraph{Key Contributions}
\begin{enumerate}
\item \textit{Width Upscaling Architecture.}
\scale\ freezes all pre-trained parameters and expands capacity with lightweight blocks inside linear modules, preserving the base computation graph and residual/attention structure\,(Section~\ref{sec:scale}).

\item \textit{Design Principles with Evidence.}
We formalize \emph{Persistent Preservation} and \emph{Collaborative Adaptation}. Theoretical analyses and preliminary studies show persistent function preservation under the preservation-oriented setup and strong plasticity under the collaborative setup\,(Section~\ref{sec:scale}).

\item \textit{Width-Upscaled Learning Methods.}
We introduce \scalepreserve\,(preservation-first), \scaleadapt\,(adaptation-first), and \scaleroute\,(advantages-of-both). \scaleroute\ routes tokens between preservation and adaptation paths, and enjoys a tighter convergence bound than standard CL\,(Section~\ref{sec:scalelearning}).

\item \textit{Empirical Validation.}
On a controlled biography task and CPT over the Korean dataset, \scale\ variants reduce forgetting on English evaluations while maintaining or improving target-domain performance; \scaleroute\ delivers the strongest stability–plasticity balance among the evaluated baselines\,(\fft, \lora, \freeze, \llamapro) under our settings\,(Section~\ref{sec:experiments}).
\end{enumerate}

%% file: 2_related_work.tex
\section{Related Work}

\paragraph{Continual Learning}
Continual Learning\,(CL) aims to adapt models to sequentially arriving tasks without catastrophic forgetting of previously acquired knowledge. Classical CL approaches include regularization\,\cite{kirkpatrick2017overcoming, zenke2017continual, aljundi2018memory}, replay\,\cite{rolnick2019experience, shin2017continual, sun2019lamol}, and parameter isolation methods\,\cite{rusu2016progressive, li2021prefix, hu2022lora, zhang2023llama}. With the rise of LLMs, CL research has shifted toward preserving the extensive range of pre-trained knowledge while incorporating new linguistic or domain-specific knowledge. However, conventional CL approaches designed for smaller models or narrow-domain tasks often struggle when applied to LLMs, due to computational and privacy constraints. Recent studies emphasize parameter-efficient fine-tuning\,(PEFT) and architectural extensions that allocate new capacity for adaptation, enabling LLMs to retain general knowledge and improve performance in target domains such as multilingual settings or specialized industries.

\paragraph{Upscaled Learning}
Model upscaling aims to enhance the capabilities of pre-trained LLMs by efficiently expanding their architectures while preserving or improving performance. This approach can be broadly categorized into the types of upscaling: depth and width upscaling. Depth upscaling increases the number of layers to capture more complex representations while leveraging existing pre-trained weights. SOLAR\,\cite{kim2023solar} and \llamapro\,\cite{llamapro} expand model depth through layer duplication and interleaving, followed by continual pre-training to recover performance. In contrast, width upscaling horizontally expands the model by increasing the hidden state dimension or the number of attention heads. Function-preserving transformations\,\cite{chen2015net2net, chen2021bert2bert} enable width upscaling to maintain consistency with the original model outputs. Building on this idea, various width upscaling approaches\,\cite{shen2022staged, yao2023masked, samragh2024scaling} along with depth upscaling initialize larger models from smaller pre-trained ones to minimize training overhead. From the perspective of continual learning, ELLE\,\cite{qin2022elle} and LOIRE\,\cite{han2025loire} apply both depth and width upscaling with function-preserving initialization to mitigate catastrophic forgetting while incorporating new domain-specific knowledge.

%% file: 3_upscaling_architecture.tex
\begin{figure*}[t!]
\centering
    \includegraphics[width=.95\textwidth]{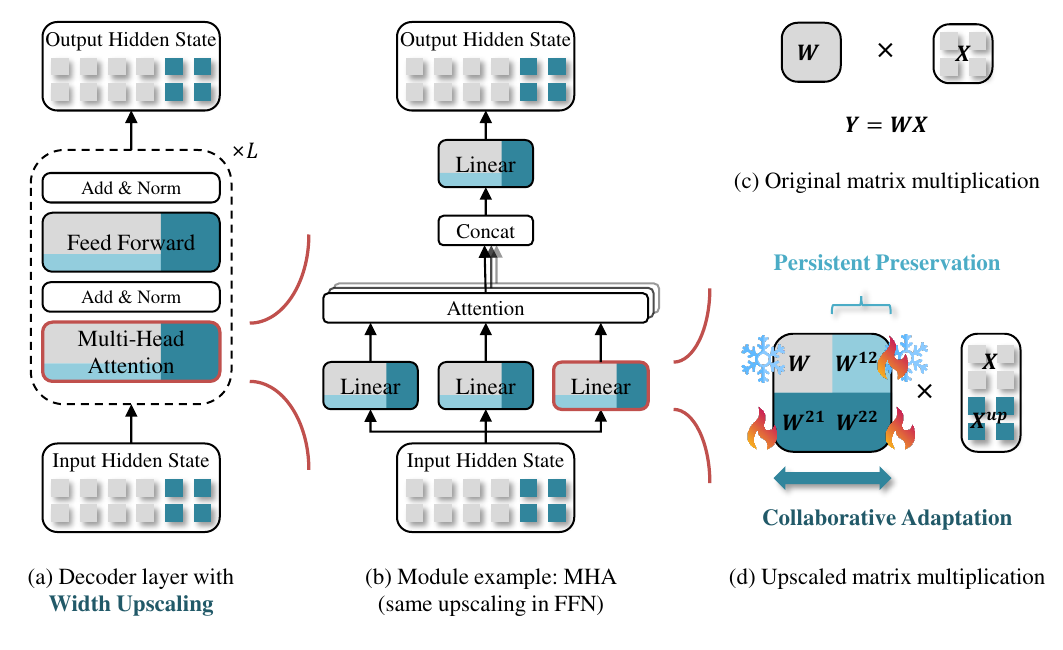}
    \caption{\textbf{Overview of \scale architecture.}}
    \label{fig:overview}
\end{figure*}

\section{Proposed Architecture: \scale}
\label{sec:scale}

In this section, we propose novel width upscaling \scale architecture, provide empirical findings, and corroborate them with theoretical analyses. Above all, an overview of \scale architecture is presented in Section \ref{sec:overview}, followed by two design principles: ① \emph{Persistent Preservation} and ② \emph{Collaborative Adaptation}. Based on \emph{Persistent Preservation} in Section~\ref{sec:preservation}, \scale zero-initializes and freezes $\bW^{12}$ to persistently preserve the original function until the end of training. In the meanwhile, based on \emph{Collaborative Adaptation} in Section~\ref{sec:adaptation}, \scale collaboratively trains certain weight blocks to effectively adapt to new domain knowledge while preserving prior knowledge.

\subsection{Overview of \scale Architecture}
\label{sec:overview}

\sukhoon{Fig 3 wrong representation in QKV projections. a Q prj. should be a single linear so that result Y up forms the new query set for new heads.}
\jinwoo{반영 했습니다.}

\sukhoon{Fig 3 no consistancy in in/out matrix representation in (a), (b), (d). The in/out matrix in (a) and (b) employs a row vector representation for hidden states, whereas (d) adopts a column vector representation. While conventional math adopts column vector representation, deep learning conventionally uses row vector representation as in BERT, GPT, Transformer papars. It may cause confusion for the readers, so need to state X $\in$ $R^{(Dim, Batch)}$, not R(B, D), if using column vector. }

Figure~\ref{fig:overview} shows an overview of the proposed width upscaling \scale architecture. For decoder layers in Figure~3(a), we upscale the width of the input hidden state as well as the dimensions of the Multi-Head Attention\,(MHA) and Feed Forward Network\,(FFN), thus upscaling the width of the output hidden state. As illustrated in Figure~3(d), all matrix multiplications $\bm{WX}$ in the MHA and FFN are expanded to their upscaled ones, formulated as: 
\begin{equation}
\label{eq:upscaled_matmul}
\begin{bmatrix}
\bW & \bW^{12} \\
\bW^{21} & \bW^{22}
\end{bmatrix}
\begin{bmatrix}
\bX \\
\bX^{up}
\end{bmatrix}
\end{equation}
where $\bW^{12}$, $\bW^{21}$, and $\bW^{22}$ denote the upscaled weight matrix blocks and $\bX^{up}$ denotes the upscaled part of the input. In particular, for MHA, this upscaling involves increasing the number of attention heads while keeping the head dimension fixed. In other words, with $\bW$ being the query, key, and value projection matrices, the upscaled outputs 
\begin{equation*}
    \bW^{21}\bX + \bW^{22}\bX^{up}
\end{equation*}
produce the corresponding query, key, and value representations for new heads.
Consequently, \scale has a structural constraint that all the weight matrices in MHA must be expanded such that the number of rows increases by an integer multiple of the head dimension. For embedding and output projection weight matrices, we upscale as follows:
\begin{equation*}
\bW_{in} \mapsto
\begin{bmatrix}
    \bW_{in} \\ \bW_{in}^{up}    
\end{bmatrix}, \quad
\bW_{out} \mapsto
\begin{bmatrix}
    \bW_{out} & \bW_{out}^{up}  
\end{bmatrix}.
\end{equation*}

\subsection{Design Principle 1: \emph{Persistent Preservation}}
\label{sec:preservation}

Appropriate weight block initialization and freeze strategy can prevent forgetting by allowing the original function $\bm{WX}$ to be persistently preserved even until the end of training, thus called \emph{Persistent Preservation}. We mainly insist that $\bW^{12}$ be zero-initialized and frozen and also suggest that $\bW^{21}$ and $\bW^{22}$ be initialized by imitating $\bW$ and $\bW^{12}$, respectively.
\sukhoon{seems to need more explanation on the trainable parts.}

\paragraph{Initialization and Freeze of $\bW^{12}$}

\begin{figure*}[t!]
    \centering
    \includegraphics[width=.55\textwidth]{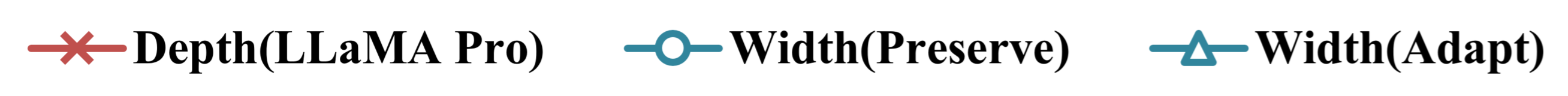} \\
    \begin{subfigure}[b]{.37\textwidth}
        \centering
        \includegraphics[width=\linewidth]{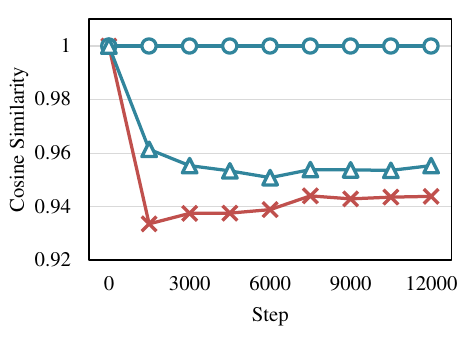}
        \caption{Representation forgetting.}
        \label{fig:depth_width_cosine}
    \end{subfigure}
    \begin{subfigure}[b]{.37\textwidth}
        \centering
        \includegraphics[width=\linewidth]{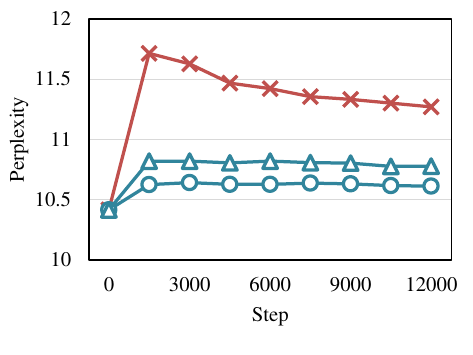}
        \caption{Forgetting perplexity.}
        \label{fig:depth_width_ppl}
    \end{subfigure}
    \caption{\textbf{Comparison of depth and width upscaling from the perspective of how (a) representation forgetting causes (b) forgetting perplexity of pre-trained English knowledge across steps for new domain adaptation.}}
    \label{fig:depth_width}
\end{figure*}

As a preliminary study with regard to function preservation, we compare depth\,(\llamapro\,\cite{llamapro}) and width\,(\scale) upscaling from the perspective of how representation forgetting causes forgetting perplexity of pre-trained English knowledge\footnote{In this section, we perform CPT on FineWeb2 Korean data subset with Llama-3.2-1B as the base model of \scale.}, as shown in Figure~\ref{fig:depth_width}. Specifically, \widthpreserve is a basic setup of \scale with zero-initialized $\bW^{12}$ which is frozen for all layers for function preservation and the rest of randomly initialized trainable weights, whereas \widthadapt is a contrastive setup of \scale including non-preserving $\bW^{12}$ which is set to be trainable for all layers.
\sukhoon{[The term "trainable weights" is not defined in the earlier part of the paper.} \jinwoo{이 섹션 이전인 Architecture Overview 섹션에서는 확장 전후만 설명하기 때문에 Initialization 이나 Trainable 에 대해 설명하지 않으려 하고, 현재 섹션에서 처음 Trainable 개념이 등장합니다. 그런데, 학습 과정에서 너무 당연한 개념이라 따로 정의가 필요한 것 같지는 않은데요. 이 코멘트 고려해서 더 의견 있으시면 구체적으로 의견 주시면 감사하겠습니다.}
\sukhoon{"is based on a basic setup" 이라는 표현을 쓰셨고, 우리의 basic setup인데 앞쪽에 미리 설명되어야 할 것 같아 쓴 코멘트 입니다.
이부분은 나중에 보시죠. 윗 부분 전체 아키를 설명을 추가하게 되면 되면 자연스럽게 추가될 설명 같습니다. 
}

Above all, Figure~\ref{fig:depth_width_cosine} shows representation forgetting measured by cosine similarity of output hidden states after the last decoder layer between step 0 and subsequent training steps. We only compare the original part of the output hidden states because the depth upscaling does not involve expanding the hidden state dimension. As intended in the experimental design, \widthpreserve never forgets the original function $\bm{WX}$, whereas \sffamily{Depth}\rmfamily(\llamapro) abruptly forgets around step 1000 and has a hard time recovering back for the subsequent steps.
This resistance to forgetting stems from the amount of function preservation. The width upscaling with zero-initialized frozen $\bW^{12}$ allows the original function to be persistently preserved until the end of training, whereas the depth upscaling inevitably perturbs the original function even at the beginning steps due to trainable upscaled layers in the middle. In contrast, \widthadapt gradually forgets representation as in Figure~\ref{fig:depth_width_cosine} and ends up exhibiting slightly higher forgetting perplexity than \widthpreserve, which is still significantly lower than \sffamily{Depth}\rmfamily(\llamapro), as in Figure~\ref{fig:depth_width_ppl}.

To corroborate our preliminary study, we analyze the function preservation of width upscaling in Theorem~\ref{thm:preserv_width_upscale} and derive that setting $\bW^{12}$ to $\bm{0}$ is necessary for all layers, which can be achieved by zero-initialization. Furthermore, from the difference between \widthpreserve and \widthadapt, we insist that freeze of $\bW^{12}$ as well as zero-initialization be necessary for \emph{Persistent Preservation}.

\sukhoon{preserve는 원본으로부터 확장 및 학습 후를 뜻할까요? 아니면 확장 전후를 뜻할까요? zero initialization 이라면 학습후 까지를 묵시적으로 뜻하기 때문에 init and fix 개념까지 들어가야 할 것 같고, 확장 전 후라면 init이 아닌 set의 같은 단어를 써야 할 것 같습니다.}
\jinwoo{Appendix 에서 지적해 주신 것처럼 Theorem 전개 과정은 학습 시점이 아니라 확장 전후 시점에 대한 내용이기 때문에, 위쪽에 처음 Theoerem 언급하는 부분에서 set 표현으로 바꾸고 ", which can be achieved by zero-initialization in terms of training" 으로 부연설명 했습니다. 다만, 이 섹션은 initialization 에 대한 내용이기 때문에, 그 이후 문장에서는 \widthpreserve 와 \widthadapt 의 성능 차이로 zero-initialization 뿐만 아니라 freeze 도 필요하다는 것을 실험적으로 주장하고 있습니다.}

\begin{theorem}
\label{thm:preserv_width_upscale}
Width-upscaled network with $\bW^{12}_{\ell}$ set to $\bm{0}$ preserves the original function for all layers $1 \le \ell \le L$.
\end{theorem}
\begin{proof}
We defer proof to Appendix \ref{sec:preserv_width_upscaling} for lack of space.
\end{proof}

\paragraph{Initialization of $\bW^{21}$ and $\bW^{22}$}

\begin{figure*}[t!]
    \centering
    \includegraphics[width=.4\textwidth]{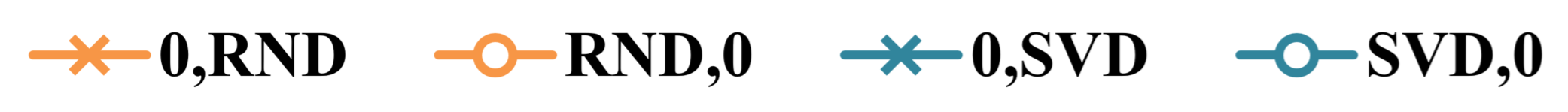} \\
    \begin{subfigure}[b]{.4\textwidth}
        \centering
        \includegraphics[width=\linewidth]{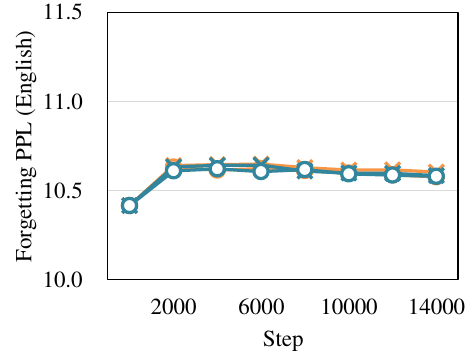}
        \caption{Forgetting perplexity on test English data.}
        \label{fig:init_ppl_eng}
    \end{subfigure}
    \begin{subfigure}[b]{.4\textwidth}
        \centering
        \includegraphics[width=\linewidth]{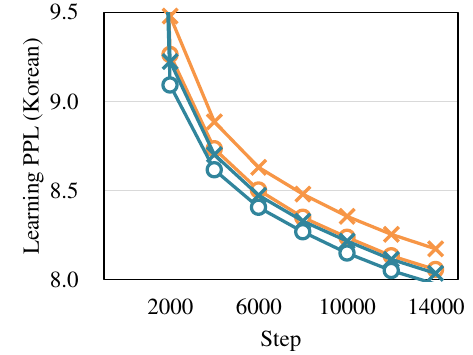}
        \caption{Learning perplexity on test Korean data.}
        \label{fig:init_ppl_kor}
    \end{subfigure}
    \caption{\textbf{Comparison of initialization pairs for $\bW^{21}$ and $\bW^{22}$. A paired name of two methods is delimited by a comma.}}
    \label{fig:init}
\end{figure*}

We also compare initialization pairs for $\bW^{21}$ and $\bW^{22}$ in Figure~\ref{fig:init}. \textbf{0} denotes zero-initialization, \textbf{RND} denotes random initialization\,\cite{he2015delving}, and \textbf{SVD} denotes dimension-reduced SVD of $\bW$ to fit the dimensions of $\bW^{21}$ or $\bW^{22}$, suggested by LESA\,\cite{yang2025lesa}. Interestingly, every initialization pair exhibits near-perfect preservation in Figure~\ref{fig:init_ppl_eng}. On the other hand, adaptation performance consistently maintains the following order in Figure~\ref{fig:init_ppl_kor}: \textbf{0,RND} $<$ \textbf{RND,0} $<$ \textbf{0,SVD} $<$ \textbf{SVD,0}. Therefore, we choose \textbf{SVD,0} as the default initialization strategy for $\bW^{21}$ and $\bW^{22}$ in this paper. Corollary~\ref{cor:irrelevant_init} provides an analysis that supports this finding.
\sukhoon{아직 작성 중인 부분 같지만.. Fig 5 (a)가 주장하는 내용은 이 section이 맞는 것 같습니다. 다만, Fig 5 (b)도 매우 중요한 만큼, 다른 색션을 만들어 설명해야 할 것 같고,}
\jinwoo{$\bW^{21}$와 $\bW^{22}$ 초기화에 대한 영향을 보이는 부분인데 같은 실험에 대해 다른 섹션을 만들어야 할 필요가 있을까요? Zero, Random, SVD 방법론 세 개를 비교하는 것 뿐이라 섹션을 나눌만큼 내용이 많지 않기도 한데, 고민을 해 보겠습니다.}
\sukhoon{SVD에 대한 내용도 설명해야 할 것 같습니다. (타 논문이지만, 개념 만이라도). }
\jinwoo{이 부분은 $\bW^{12}$와 비교했을 때 중요한 내용이 아니라고 생각되는데 논의가 필요해 보입니다. 일단, Appendix 에 내용을 정리하고 안내문구를 넣겠습니다.}

\sukhoon{많은 고민을 하지 않고 읽으면 Fig4의 (A)와 모순인 것 처럼 보입니다. 평가대상이 hidden이 다른것을 넣거나..실험설계 내용을 조금이라도 넣으면 되겠네요. }

\begin{corollary}
\label{cor:irrelevant_init}
$\bW^{21}_{\ell}$ and $\bW^{22}_{\ell}$ can be initialized to other than $\bm{0}$ for new task adaptation without disrupting the original function because $\bW^{21}_{\ell}$ and $\bW^{22}_{\ell}$ values are irrelevant to the function preservation.
\end{corollary}

\subsection{Design Principle 2: \emph{Collaborative Adaptation}}
\label{sec:adaptation}

By collaboratively training certain weight blocks, \scale can effectively adapt to new domain knowledge, thus called \emph{Collaborative Adaptation}. We mainly suggest that weight blocks in certain layers or in certain modules be collaboratively trained.


\paragraph{Collaborative Layers}

\begin{figure}[t!]
\centering
    \includegraphics[width=.9\linewidth]{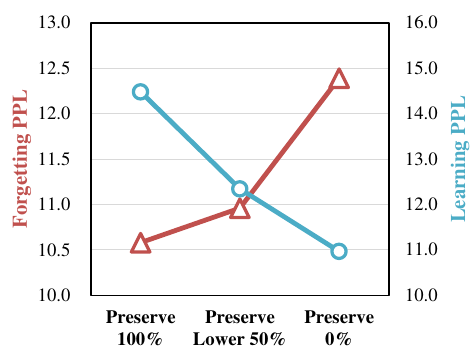}
    \caption{\textbf{Trade-off of forgetting and learning PPL according to the amount of preserving lower layers $L_{fp}$.}}
    \label{fig:L_fp}
\end{figure}
\begin{figure}[t!]
\centering
    \includegraphics[width=.9\linewidth]{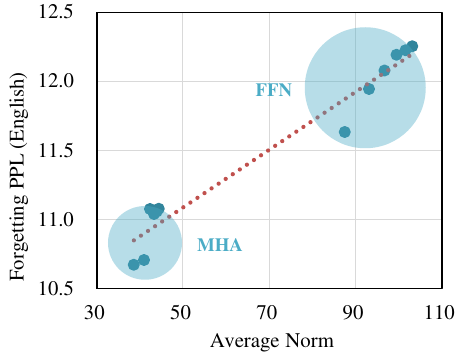}
    \caption{\textbf{Near-linear scalability of forgetting perplexity to upscaled weight norm.}}
    \label{fig:norm_ppl}
\end{figure}

We introduce a preliminary study to show a collaborative property of \scale that leads to a trade-off of forgetting and learning. Since frozen $\bW^{12}$ for the \emph{Persistent Preservation} strictly prevents collaboration between $\bX$ and $\bX^{up}$, it inherently lacks learning capacity. Therefore, we initialize $\bW^{12}$ to be $\bm{0}$ and permit certain amount of $\bW^{12}$ blocks only in upper layers to be trainable, thus collaborating with the original weight block $\bW$ in upper layers while $\bW^{12}$ blocks in lower layers are still preserved. Figure~\ref{fig:L_fp} displays a trade-off of forgetting and learning PPL according to $L_{fp}$ which is the number of function preserving $\bW^{12}$ blocks in lower layers. It should be also noted that forgetting happens exponentially with decreasing $L_{fp}$. Proposition~\ref{pro:preserv_width_upscaling} and Corollary~\ref{cor:L_fp} provide an analysis that supports this finding.

\begin{proposition}
\label{pro:preserv_width_upscaling}
The accumulated output shift of the width-upscaled residual network with function-preserving lower layers $1 \le \ell \le L_{fp}$ and non-preserving upper layers $L_{fp} < \ell \le L$  is bounded by Eq.~\eqref{eq:bound_preserv_width_upscaling}.

\begin{equation}
\begin{multlined}
\norm{ \tilde{\bX}^{UP}_{L} - \bX^{UP}_{L} } \\
\le \p{ L - L_{fp} } \epsilon \p{ 1 + \delta_{np} }^{L - 1} \p{ \frac{1 + \delta_{fp}}{1 + \delta_{np}} }^{L_{fp}} \norm{ \bX^{UP}_{0} }
\end{multlined}
\end{equation}
where $\tilde{\bX}^{UP}_{L}$ denotes updated output of width-upscaled residual network as defined by Definition~\ref{def:updated_width_upscaling} and $\delta_{fp}$ and $\delta_{np}$ denotes upper bound of norm of upscaled weight matrix for function-preserving lower layers and non-preserving upper layers, respectively, under Assumption~\ref{assume:smaller_weight_norm}.
\end{proposition}
\begin{proof}
We defer the proofs in this section to Appendix \ref{sec:forgetting_analysis}.
\end{proof}

\begin{corollary}
\label{cor:L_fp}
Forgetting increases exponentially with decreasing $L_{fp}$, number of function preserving $\bW^{12}_{\ell}$ blocks in lower layers $1 \le \ell \le L_{fp}$.
\end{corollary}

\paragraph{Collaborative Modules}

In addition to collaborative layers, we further look into which module is more effective in preservation and collaboration. Specifically, based on function-preserving lower half layers, i.e., $L_{fp}=L/2$, we train two cases for an epoch with $\bW^{12}$ blocks in either MHA or FFN module to be trainable, thus collaborating with $\bW$ blocks only in the module's upscaled weights. Figure~\ref{fig:norm_ppl} shows that both cases in total exhibit near-linear scalability of forgetting perplexity to average norm of all upscaled weights. The case with FFN module exhibits high forgetting perplexity due to its large intermediate dimension size and, considering baseline forgetting perplexity of LLaMA-3.2-1B around 10.4, the converged perplexity suggests it is not a good choice. On the other hand, the other case with MHA module converges within smaller error bound and exhibits only slightly higher perplexity compared to the baseline, which makes it a good choice for collaborative module.

\sukhoon{실험설명이 좀더 되야 할 것 같습니다. 특히 Avg Norm을 vary하는 컨트롤 요소가 뭐인지..  점들의 점들의 avg norm 간격이 일정하지 않아서, 단순 레이어 단위로 trainable 설정한 것 같지는 않고.. average면 sum/N 인데 N이 뭘까.. 윗쪽 exponential fogetting 과 이곳 linear가 서로 상충 되지 않는 다른 trainable 설정인걸까..}



%% file: 4_upscaled_learning.tex
\section{Proposed Learning Methods}
\label{sec:scalelearning}

\emph{Persistent Preservation} and \emph{Collaborative Adaptation}, two design principles of \scale, suggest complementary methods for upscaled continual learning. Additionally, taking only the advantages of both methods could possibly improve preservation and adaptation simultaneously. To this end, we propose three novel learning methods: ① \scalepreserve, ② \scaleadapt,  and ③ \scaleroute.

\begin{itemize}
\item \scalepreserve is the preservation-first method such that $\forall \ell$, $\bW^{12}_{\ell}$ is initialized to $\bm{0}$ and not trainable, as defined by Eq.~\eqref{eq:Z_preserve}.
\begin{equation}
\label{eq:Z_preserve}
\bZ^{UP}_{preserve} \triangleq \bZ_{preserve} + \bZ^{up}_{preserve}
\end{equation}
where $\bZ=\bW_{out} \bX_{L}$ and $\bZ^{up}=\bW^{up}_{out} \bX^{up}_{L}$ denote output logits for original and upscaled part, respectively.

\item \scaleadapt, on the other hand, is the adaptation-first method such that $\forall \ell$, $\bW^{12}_{\ell}$ is initialized to $\bm{0}$ and trainable, as defined by Eq.~\eqref{eq:Z_adapt}.
\begin{equation}
\label{eq:Z_adapt}
\bZ^{UP}_{adapt} \triangleq \bZ_{adapt} + \bZ^{up}_{adapt}
\end{equation}

\item \scaleroute is the advantages-of-both method that aims at maximizing preservation and adaptation simultaneously and thus mitigating the performance trade-off illustrated in Figure~\ref{fig:L_fp}. Note that the core difference between preservation and adaptation of \scale lies in trainability of the upscaled weights $\bW^{12}$. Therefore, delicate utilization of trainable computation paths is key to the optimal performance. \scaleroute routes the best computation paths to either \scalepreserve or \scaleadapt, as defined by Eq.~\eqref{eq:Z_route}.
\begin{equation}
\label{eq:Z_route}
\bZ^{UP}_{route} \triangleq \begin{dcases}
    \begin{multlined}
        \p{ \bZ_{preserve} + \bZ_{adapt} } / 2 + \bZ^{up}_{preserve} \\
        \text{if cosine}\p{\bZ_{preserve}, \bZ_{adapt}} > \uptau
    \end{multlined} \\
    \bZ^{UP}_{adapt} \hfill \text{otherwise}
\end{dcases}
\end{equation}
where cosine$\p{\cdot}$ denotes cosine similarity between two logits, which acts as a router controlled by minimum threshold $\uptau$. In brief, if logits of a token for $\bZ_{preserve}$ and $\bZ_{adapt}$ are similar, it routes the token to \scalepreserve. Otherwise, it takes an adaptation opportunity by routing the token to \scaleadapt. $\bZ_{preserve}$ is approximated by $\p{\bZ_{preserve} + \bZ_{adapt}}/2$ due to better trainability of $\bZ_{adapt}$. Note that according to Section \ref{sec:preservation}, the computation paths of \scaleadapt override those of \scalepreserve, thereby allowing \scaleadapt to produce logits of \scalepreserve as long as $\bW^{12}$ being computed as $\bm{0}$ matrix. Therefore, it still requires a single forward pass with slight extra computation to produce both logits at any training step.
\end{itemize}

Theorem~\ref{thm:logit_routed_cl} provides an analysis that supports the superiority of \scaleroute.

\begin{theorem}
\label{thm:logit_routed_cl}
\emph{Routing-based Continual Learning} achieves lower convergence than standard \emph{Continual Learning}.
\end{theorem}
\begin{proof}
The proof is deferred to Appendix \ref{sec:convergence_analysis}.
\end{proof}

%% file: 5_experiments.tex
\section{Experiments}
\label{sec:experiments}

\subsection{Experimental Setup}

\begin{table}[t!]
    \centering
    \begin{tabular}{lcc}
        \hline
        Model & Learning rate \\
        \hline
        \fft & $1 \times 10^{-5}$ \\
        \lora & $1 \times 10^{-5}$ \\
        \freeze & $1 \times 10^{-5} $ \\
        \llamapro & $2 \times 10^{-4}$ \\
        \scale & $1 \times 10^{-3}$ \\
        \hline
    \end{tabular}
    \caption{\textbf{Learning rates of continual pre-training on the Korean dataset.}}
    \label{tab:main-setting}
\end{table}

\paragraph{The Biography Dataset}
We reproduce the controlled experiment conducted in \cite{zheng2025spurious} to compare forgetting between existing continual learning methods and our proposed methods, especially \scaleroute. The experimental dataset, \textit{Biography Dataset}\footnote{https://github.com/zzz47zzz/spurious-forgetting}, consists of 200,000 synthetic individuals, each characterized by a name and six attributes: birthday, birth city, university attended, major, company name, and company city. For each individual, the dataset is divided into pre-training and fine-tuning data. As in \cite{zheng2025spurious}, our continual learning setting is constructed through three stages. The model is initially pre-trained on the first 100,000 individuals of the synthetic Biography Dataset, followed by fine-tuning on QA data corresponding to the first 50,000 individuals\,(Task~0). We then apply an upscaling method and fine-tune new QA data from a previously unseen 20,000 individuals\,(Task~1). During learning Task 1, we monitor the decline of the model's performance on Task 0, using \textit{hard first-token accuracy}, a metric that measures whether the model's top-predicted first token matches the correct token. 

We utilize the Pythia-160M \cite{biderman2023pythia} architecture as our backbone model. For \scaleroute, we upscale both the hidden dimension and the FeedForward dimension by 128, and train $\bW^{12}$ only for the last 12th layer. In order to match the number of trainable parameters in \scaleroute, \llamapro expands the number of layers from 12 to 16. We note that \llamapro relies on the LLaMA architecture, where the output weight matrices are zero-initialized in the expanded block to preserve the output from the initial model. In contrast, since Pythia adopts a GPT-NeoX \cite{gpt-neox-library} architecture, a different set of the output weight matrices should be zero-initialized. We use the same hyperparameter settings as in \cite{zheng2025spurious}, except that, during Task 1 learning, \scaleroute and \llamapro are trained with an increased learning rate of $5 \times 10^{-5}$, which is ten times larger than the original setting. This modification is necessary because
both methods freeze a substantial portion of the parameters, and therefore an increased learning rate is required to ensure the performance on Task 1. The experiments are executed on an NVIDIA H100 80GB GPU.



\paragraph{Continual Pre-training}
In order to investigate forgetting phenomena, we constrain the training data to the Korean subset of FineWeb2 \cite{penedo2025fineweb2pipelinescale}, a 60-billion-token Korean web data filtered from Common Crawl. We deliberately exclude data from other domains, since the presence of English corpus can induce data-replay effects, which in turn hinder a precise comparison among upscaling methods.

For each method, we initialize our base model with LLaMA3.2-1B and perform continual pre-training on the Korean dataset for one epoch, using a batch size of 512, a sequence length of 8192, and a linear learning rate schedule with a warm-up ratio of 6\%. Our \scale methods upscale both the hidden dimension and the FeedForward dimension by 256 and 1024, respectively. We note that the base model is configured with a head dimension of 64 and uses the Grouped-Query Attention(GQA) with 4 KV projections, and therefore upscaling the hidden dimension by 256 represents a minimal upscaling. For \scaleadapt and \scaleroute, we choose $L_{fp} = 3$. We also manually configure the hyperparameters of \llamapro and \lora to match the number of trainable parameters in \scale methods. We expand the number of layers from 16 to 20 for \llamapro, and use a rank of 256 and target all weight matrices in the MHA and FFN for \lora. \freeze \cite{zheng2025spurious} refers to freezing all components in the bottom three layers of the model, including the input embedding layer. Finally, we set different learning rates due to the trade-off between learning and forgetting. As shown in Table~\ref{tab:main-setting}, we adjust the learning rate individually for each experiment to achieve comparable learning performance, allowing us to fairly compare model's forgetting under similar learning conditions. For \freeze, we employ the same learning rate of \fft, as in the original paper. All experiments are executed on 8 NVIDIA H100 80GB GPUs.

\begin{figure*}[t!]
    \centering
    \begin{subfigure}[b]{.34\textwidth}
        \centering
        \includegraphics[width=\linewidth]{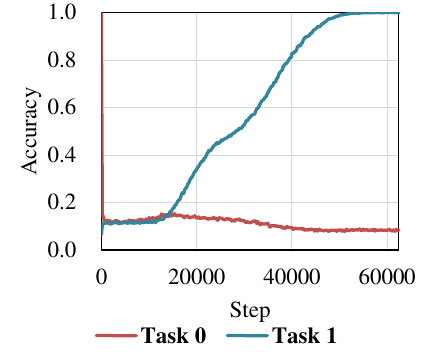}
        \caption{\fft.}
        \label{fig:biography_fft}
    \end{subfigure}
    \begin{subfigure}[b]{.315\textwidth}
        \centering
        \includegraphics[width=\linewidth]{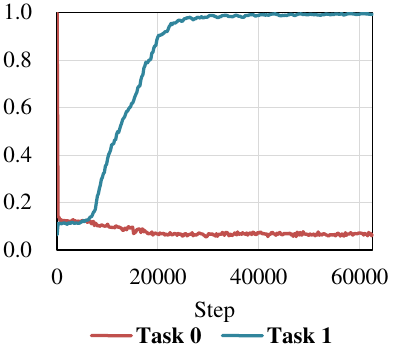}
        \caption{\llamapro.}
        \label{fig:biography_llamapro}
    \end{subfigure}
    \begin{subfigure}[b]{.315\textwidth}
        \centering
        \includegraphics[width=\linewidth]{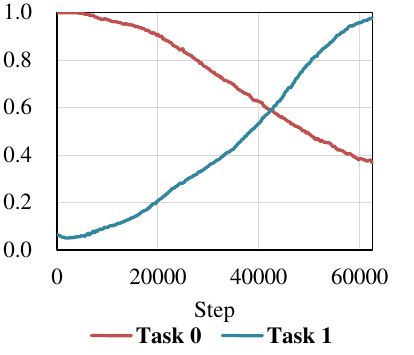}
        \caption{\scaleroute.}
        \label{fig:biography_scale}
    \end{subfigure}
    \caption{\textbf{Continual adaptation to biography Task 1 while preserving Task 0.}}
    \label{fig:biography_trend}
\end{figure*}

\subsection{Results and Analysis}

\paragraph{The Biography Dataset Results}
For the Biography Dataset experiment, we present the accuracy for Task 0 and Task 1 during Task 1 learning in Figure~\ref{fig:biography_trend}. We observe that for \fft and \llamapro, the accuracy for Task 0 sharply drops to approximately 15\% only after 200 steps, whereas for \scaleroute it remains at 100\% throughout the first 4000 steps of Task 1 learning. Furthermore, for \scaleroute the final accuracy for Task 0 is 36.9\% which is much higher compared to \fft and \llamapro, highlighting its robustness against forgetting.

Another notable observation is that in Figure~\ref{fig:biography_scale}, the accuracy curves of \scaleroute for Task 0 and Task 1 gradually decrease and increase, respectively, showing smooth transitions without drops or spikes during Task 1 learning. This indicates that by varying the number of collaborative layers, the trade-off between forgetting and learning can be controlled in our architecture-based method, rather than a data replay-based method, aligning it with its learning objective.

\paragraph{Continual Pre-Training Results}
We first analyze the perplexity on 30K samples of FineWeb-Edu and on the test split of the Korean subset of FineWeb2. As shown in Figure~\ref{fig:main_ppl_kor}, the perplexity on the Korean test data is almost identical across all methods except \scalepreserve, which is consistent with our intended design. In contrast, Figure~\ref{fig:main_ppl_eng} shows that our \scale methods achieve lower perplexity on the English test data than the other methods. Compared to \llamapro, \scale has the lower increase of perplexity on the English test data in early stage of training, and therefore it can be regarded as a more stable approach for Continual Pre-Training. We also observe that the discrepancy of the perplexity between \scalepreserve and \scaleadapt arises more from learning than from forgetting, proving that training $\bW^{12}$ more strongly affects learning than forgetting.

\begin{figure*}[t!]
    \centering
    \includegraphics[width=\textwidth]{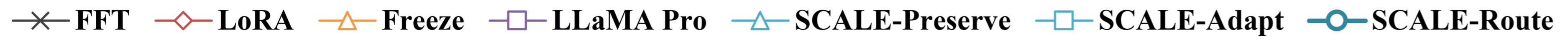} \\
    \begin{subfigure}[b]{.45\textwidth}
        \centering
        \includegraphics[width=\linewidth]{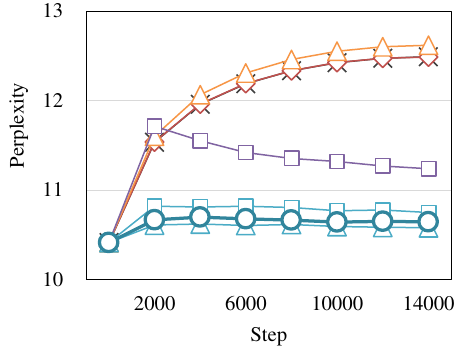}
        \caption{Forgetting perplexity on test English data.}
        \label{fig:main_ppl_eng}
    \end{subfigure}
    \begin{subfigure}[b]{.45\textwidth}
        \centering
        \includegraphics[width=\linewidth]{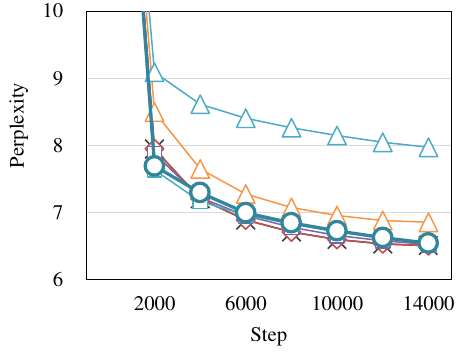}
        \caption{Learning perplexity on test Korean data.}
        \label{fig:main_ppl_kor}
    \end{subfigure}
    \caption{\textbf{Performance comparison of upscaling methods trained on the Korean subset of FineWeb2 for one epoch.}}
    \label{fig:main}
\end{figure*}

Furthermore, we evaluate our \scale methods with the English and Korean benchmarks. Evaluations are conducted using Eleuther AI Language Model Evaluation Harness\footnote{https://github.com/EleutherAI/lm-evaluation-harness}
and in a zero shot setting. Due to the lack of instruction-following capability of pre-trained models, we select only a very limited set of Korean benchmarks, KoBEST \cite{jang2022kobest}, for evaluation. The results are presented in Table~\ref{tab:main-performance}. We find that all \scale methods preserve the original capabilities on English benchmarks, outperforming other methods. Although \scaleroute obtains some improvement on the Korean benchmarks, it achieves only marginal improvement compared to \fft and \lora. However, as \scaleadapt and \scaleroute outperform \scalepreserve on the Korean benchmarks, expanding the training scope of $\bW^{12}$ could yield further improvements while preserving the original capabilities.

\begin{table*}[t!]
    \centering\resizebox{\textwidth}{!}{
    \begin{tabular}{lccccccccccc}
        \hline
        Model & \multicolumn{6}{c}{English} &  & \multicolumn{4}{c}{Korean} \\
        \cline{2-7} \cline{9-12}
         & ARC & HellaSwag & MMLU & TruthfulQA & Winogrande & Avg. &  & KB BoolQ & KB COPA & KB HellaSwag & Avg. \\
        \hline
        Llama-3.2-1B & 36.60 & 63.66 & 36.76 & 37.73 & 60.93 & 47.14 &  & 49.86 & 53.00 & 50.60 & 51.15 \\
        \hline
        \fft & 32.68 & 57.32 & 26.01 & 38.82 & 59.12 & 42.79 &  & 52.14 & 65.30 & 54.20 & 57.21 \\
        \lora & 32.85 & 57.17 & 25.76 & 38.99 & 58.17 & 42.59 & & 51.85 & 64.90 & 55.00 & 57.25 \\
        \hline
        \llamapro & 34.22 & 60.61 & 34.42 & 36.20 & 57.30 & 44.55 &  & 51.42 & 63.20 & 52.80 & \textbf{55.81} \\
        \freeze & 31.23 & 56.59 & 26.43 & 38.67 & 59.27 & 42.44 & & 50.50 & 60.10 & 53.40 & 54.67 \\
        \scalepreserve & 36.52  & 62.75 & 33.79 & 37.89 & 59.51 & \textbf{46.09} & & 52.42 & 57.60 & 49.40 & 53.14 \\
        \scaleadapt & 34.81 & 61.85 & 35.22 & 38.25 & 60.62 & \textbf{46.15} &  & 50.36 & 63.10 & 51.80 & \textbf{55.09} \\
        \scaleroute & 35.84 & 61.72 & 36.31 & 37.50 & 61.09 & \textbf{46.49} & & 51.50 & 63.80 & 51.20 & \textbf{55.50} \\
        \hline
    \end{tabular}
    }
    \caption{\textbf{Performance comparison of upscaling methods trained on the Korean subset of FineWeb2 for one epoch.}}
    \label{tab:main-performance}
\end{table*}

%% file: 6_conclusion.tex
\section{Conclusion}
This work presented \scale, a practical recipe for architectural expansion that enables stable continual pre-training without altering a model’s core computation graph. By growing width inside linear submodules and freezing the original parameters, \scale adds representational capacity while keeping the pre-trained function intact. Two design principles—\emph{Persistent Preservation} and \emph{Collaborative Adaptation}—let practitioners dial in the desired balance: preservation-oriented setups deliver extreme resistance to forgetting, while collaborative setups unlock strong plasticity by training a targeted subset of expansion components, preferably in upper layers and within attention when stability is paramount.

We instantiated these principles through three learning variants. \scalepreserve offers a preservation-first baseline with strong stability. \scaleadapt emphasizes plasticity by enabling collaboration throughout the network. \scaleroute combines both by routing tokens between preservation and adaptation paths using a similarity criterion, exposing both behaviors in a single forward pass with small overhead. Empirically, width upscaling outperforms depth expansion on the biography task, achieving significantly higher retention and competitive adaptation. In continual pre-training on Korean web data, \scale variants reduce forgetting on English evaluations while matching or improving target-language learning; among them, \scaleroute consistently achieves the best stability–plasticity balance. Theoretical results clarify why preservation holds under our initialization and freezing scheme and why routing confers a convergence advantage.

Limitations include evaluation at modest scales and on a constrained set of domains, a simple static routing threshold, and conservative choices about which layers and modules collaborate. Future work includes scaling to larger backbones and longer training horizons, adaptive selection of collaborative layers and modules, richer routing policies, integration with parameter-efficient tuning and retrieval, and broader multilingual and domain-specialized CPT. Together, these directions position width-upscaled continual learning as a reliable path beyond brute-force scaling.

%% file: 8_appendix.tex
\appendix

\section{Width Upscaling}
\label{sec:width_upscaling}

We consider a simplified decoder model that retains residual connections while omitting modules such as MHA\,(multi-head attention) or FFN\,(feed-forward networks). This abstraction enables us to concentrate on the dynamics of residual layers, particularly their roles in balancing \textbf{preservation} and \textbf{adaptation} of representations. Such model reduction is consistent with a broader line of theoretical work that introduces simplifications to isolate core mechanisms, as highlighted in Remark~\ref{remark:abstraction}
\begin{remark}
\label{remark:abstraction}
Theoretical analysis of the full Transformer architecture is notoriously challenging due to the intricate interaction of its components. Recent studies have therefore adopted simplified settings---such as single-layer or shallow Transformers\,\cite{li2023theoretical}, restricted weight structures\,\cite{boix2023transformers}, and attention-free or linearized variants\,\cite{ahn2023linear,zhai2021attention}---to obtain tractable insights. Our focus on a residual-only decoder follows this tradition, motivated by growing evidence that residual connections form the backbone of stable signal propagation and representation transport in Transformers\,\cite{he2023deep,qin2025convergence,deb2025information}
\end{remark}

\begin{definition}
\label{def:residual}
(Residual Network) \\
\begin{equation}
\bX_{\ell} \triangleq \p{ \bW_{\ell} + \bm{I} } \bX_{\ell-1} \quad \p{ 1 \leq \ell \leq L }
\end{equation}
where each $\ell$-th layer outputs hidden states $\bX_{\ell} \in \mathbb{R}^{d}$ and has a weight matrix $\bW_{\ell} \in \mathbb{R}^{d \times d}$, $\bX_0$ denotes embedding input, and $\bm{I} \in \mathbb{R}^{d \times d}$ denotes the identity matrix.
\end{definition}

Next, we upscale the width of the residual network, as defined in Definition \ref{def:width_upscaling}. It follows general formulations in upscaled learning literature\,\cite{stacking,shen2022staged,samragh2024scaling}.

\begin{definition}
\label{def:width_upscaling}
(Residual Network with Width Upscaling)

\begin{equation}
\label{eq:width_upscaling}
\bX^{UP}_{\ell} \triangleq \p{ \bW^{UP}_{\ell} + \bm{I}^{UP} } \bX^{UP}_{\ell-1} \quad \p{ 1 \leq \ell \leq L }
\end{equation}
where $\bX_{\ell} \,\p{ 0 \leq \ell \leq L }$ is upscaled on dimension of width to
$\bX^{UP}_{\ell} \triangleq \begin{bmatrix}
\bX_{\ell} \\
\bX^{up}_{\ell}
\end{bmatrix}
\in \mathbb{R}^{ \p{ d+d_{up} } }$
, $\bW_{\ell} \,\p{ 1 \leq \ell \leq L }$ is upscaled accordingly to
$\bW^{UP}_{\ell} \triangleq \begin{bmatrix}
\bW_{\ell} & \bW^{12}_{\ell} \\
\bW^{21}_{\ell} & \bW^{22}_{\ell}
\end{bmatrix}
\in \mathbb{R}^{ \p{ d+d_{up} } \times \p{ d+d_{up} } }$, and $\bm{I}$ is also upscaled to
$\bm{I}^{UP} \triangleq \begin{bmatrix}
\bm{I} & \bm{0} \\
\bm{0} & \bm{I}^{up}
\end{bmatrix}
\in \mathbb{R}^{ \p{ d+d_{up} } \times \p{ d+d_{up} } }$.

\end{definition}

\section{Function-Preserving Width Upscaling}
\label{sec:preserv_width_upscaling}

\textbf{Function preservation} serves as a foundational principle in model upscaling\,\cite{shen2022staged, wang2023learning, yao2023masked,samragh2024scaling,qin2022elle,han2025loire}, aiming to ensure that the functional behavior of the model remains unchanged despite architectural modifications. 

Formally, $F$ denotes the original function\,(e.g., an end-to-end function or just a layer), taking $\bX$ as the input\,(e.g., input tokens or input hidden states). By transforming or expanding $F\p{ \bX }$, we can upscale $F: \bX \to \bm{Y}$ to $\mathcal{F}: \bX \times \bm{U} \to \bm{Y} \times \bm{V}$. Then, the objective of function preservation is to satisfy Eq.~\eqref{eq:FuncPreserve}.

\begin{equation}
\label{eq:FuncPreserve}
\forall \bX, \pi_{\bm{Y}} \p{ \mathcal{F}\p{ \bX, \bm{U} } } = F\p{ \bX }
\end{equation}
where $\pi_{\bm{Y}}$ denotes a projection function to $\bm{Y}$. That is, $F$ is preserved in $\mathcal{F}$.

This formulation enables model upscaling by allowing new knowledge to be learned, while ensuring that the original knowledge remains unforgotten.

To satisfy function preservation in the width-upscaled residual network in Definition \ref{def:width_upscaling}, it is the simplest to set $\bW^{12}_{\ell}$ to $\bm{0}$, as defined in Definition \ref{def:preserv_width_upscaling}.

\begin{definition}
\label{def:preserv_width_upscaling}
(Residual Network with Function-Preserving Width Upscaling).
\begin{equation}
\label{eq:preserv_width_upscaling}
\bX^{UP}_{\ell} \triangleq \p{ \begin{bmatrix}
\bW_{\ell} & \bm{0} \\
\bW^{21}_{\ell} & \bW^{22}_{\ell}
\end{bmatrix} + \bm{I}^{UP} } \bX^{UP}_{\ell-1} \quad \p{ 1 \leq \ell \leq L }
\end{equation}
\end{definition}

\noindent \textbf{Theorem~\ref{thm:preserv_width_upscale}.} \textit{
Width-upscaled network with $\bW^{12}_{\ell}$ set to $\bm{0}$ preserves the original function $\bX_{\ell} = \p{ \bW_{\ell} + \bm{I} } \bX_{\ell-1}$ for all layers $1 \leq \ell \leq L$.}

\begin{proof}
From Eq.~\eqref{eq:width_upscaling}, we can express the original function as $F\p{ \bX_{\ell-1} } = \p{ \bW_{\ell} + \bm{I} } \bX_{\ell-1}$ and its width-upscaled function as $\mathcal{F}\p{ \bX^{UP}_{\ell-1} } = \begin{bmatrix}
\bW_{\ell} + \bm{I} & \bW^{12}_{\ell} \\
\bW^{21}_{\ell} & \bW^{22}_{\ell} + \bm{I}^{up}
\end{bmatrix} \begin{bmatrix}
\bX_{\ell-1} \\
\bX^{up}_{\ell-1}
\end{bmatrix}$.
By setting $\bW^{12}_{\ell}$ of $\mathcal{F}\p{ \bX^{UP}_{\ell-1} }$ to $\bm{0}$, it can be easily shown by Eq.~\eqref{eq:preserv_width_upscaling_derive} that Eq.~\eqref{eq:preserv_width_upscaling} satisfies function preservation because for all $\bX^{UP}_{\ell-1}$, there exists a projection function $\pi_{\bm{Y}}\p{ \mathcal{F}\p{ \bX^{UP}_{\ell-1} } } = F\p{ \bX_{\ell-1} }$.
\begin{equation}
\begin{aligned}
\label{eq:preserv_width_upscaling_derive}
\mathcal{F}\p{ \bX^{UP}_{\ell-1} } &= \begin{bmatrix}
\bW_{\ell} + \bm{I} & \bm{0} \\
\bW^{21}_{\ell} & \bW^{22}_{\ell} + \bm{I}^{up}
\end{bmatrix} \begin{bmatrix}
\bX_{\ell-1} \\
\bX^{up}_{\ell-1}
\end{bmatrix} \\
&= \begin{bmatrix}
\p{ \bW_{\ell} + \bm{I} } \bX_{\ell-1} \\
\bW^{21}_{\ell}\bX_{\ell-1} + \p{ \bW^{22}_{\ell} + \bm{I}^{up} }\bX^{up}_{\ell-1}
\end{bmatrix} \\
&= \begin{bmatrix}
F\p{ \bX_{\ell-1} } \\
\bW^{21}_{\ell}\bX_{\ell-1} + \p{ \bW^{22}_{\ell} + \bm{I}^{up} }\bX^{up}_{\ell-1}
\end{bmatrix}
\end{aligned}
\end{equation}

From the recursion of layers in Eq~\eqref{eq:preserv_width_upscaling_derive}, function preservation of the end-to-end network is also satisfied.
\end{proof}





\noindent \textbf{Corollary~\ref{cor:irrelevant_init}.} \textit{$\bW^{21}_{\ell}$ and $\bW^{22}_{\ell}$ can be initialized to other than $\bm{0}$ for new task adaptation without disrupting the original function because $\bW^{21}_{\ell}$ and $\bW^{22}_{\ell}$ values are irrelevant to the function preservation.}

\begin{proof}
Immediate from the proof of Theorem~\ref{thm:preserv_width_upscale}.
\end{proof}

\section{Forgetting Analysis}
\label{sec:forgetting_analysis}

In this section, we analyze forgetting from the perspective of accumulated output shift. First of all, we make the following assumptions, as in many other relevant studies\,\cite{zheng2025spurious,wang2021mesh,nguyen2019transformers,black2022gpt}.

\begin{assumption}
\label{assume:small_weight_norm}
(Small Weight Norm). For every layer $\ell$, the norm of its upscaled weight matrix is bounded by a small constant $\delta > 0$, i.e., $\norm{ \bW^{UP}_{\ell} } \le \delta$.
\end{assumption}

\begin{assumption}
\label{assume:small_grad_norm}
(Small Gradient Norm). For every layer $\ell$, the norm of its upscaled gradient matrix is bounded by a small constant $\epsilon > 0$, i.e., $\norm{ \Delta\bW^{UP}_{\ell} } \le \epsilon$.
\end{assumption}

\begin{definition}
\label{def:updated_width_upscaling}
(Updated Output of Width-Upscaled Residual Network) By updating the upscaled weight matrix $\bW^{UP}_{\ell}$ in Eq.~\eqref{eq:width_upscaling} to $\tilde{\bW}^{UP}_{\ell} \triangleq \bW^{UP}_{\ell} + \Delta\bW^{UP}_{\ell}$, the corresponding updated output is defined as Eq.~\eqref{eq:updated_width_upscaling}.

\begin{equation}
\label{eq:updated_width_upscaling}
\tilde{\bX}^{UP}_{\ell} \triangleq \p{ \bW^{UP}_{\ell} + \Delta\bW^{UP}_{\ell} + \bm{I}^{UP} } \tilde{\bX}^{UP}_{\ell-1} \quad \p{ 1 \leq \ell \leq L }
\end{equation}
\end{definition}

Based on Assumptions \ref{assume:small_weight_norm} and \ref{assume:small_grad_norm}, we analyze the output shift bound of the width upscaling from Definition \ref{def:width_upscaling} in Proposition~\ref{pro:width_upscaling} and extend it to the function-preserving width upscaling from Definition \ref{def:preserv_width_upscaling} in Proposition~\ref{pro:preserv_width_upscaling}.
Note that this is also an extension of \textit{Proposition~4.9 from \freeze}\,\cite{zheng2025spurious}.

\begin{proposition}
\label{pro:width_upscaling}
The accumulated output shift for all layers $1 \le \ell \le L$ of the residual network with width upscaling is bounded by Eq.~\eqref{eq:bound_width_upscaling}.

\begin{equation}
\label{eq:bound_width_upscaling}
\norm{ \tilde{\bX}^{UP}_{L} - \bX^{UP}_{L} } \le L \epsilon \p{ 1 + \delta }^{L - 1} \norm{ \bX^{UP}_{0} }
\end{equation}
\end{proposition}

\begin{proof}
We begin by deriving accumulated output in Eq.~\eqref{eq:acc_output} and accumulated output after a learning step in Eq.~\eqref{eq:acc_updated_output} from recursion of Eq.~\eqref{eq:width_upscaling} and Eq.~\eqref{eq:updated_width_upscaling}, respectively.
\begin{gather}
\label{eq:acc_output}
\bX^{UP}_{L} = \prod^{L}_{\ell=1}\p{ \bW^{UP}_{\ell} + \bm{I}^{UP} } \bX^{UP}_{0} \\
\label{eq:acc_updated_output}
\tilde{\bX}^{UP}_{L} = \prod^{L}_{\ell=1}\p{ \bW^{UP}_{\ell} + \Delta\bW^{UP}_{\ell} + \bm{I}^{UP} } \bX^{UP}_{0}
\end{gather}

By taking difference between Eq.~\eqref{eq:acc_output} and Eq.~\eqref{eq:acc_updated_output}, we have Eq.~\eqref{eq:acc_diff}.
\begin{equation}
\label{eq:acc_diff}
\begin{multlined}
\tilde{\bX}^{UP}_{L} - \bX^{UP}_{L} = \left( \prod^{L}_{\ell=1}(\bW^{UP}_{\ell} + \Delta\bW^{UP}_{\ell} + \bm{I}^{UP}) \right. \\
\left. - \prod^{L}_{\ell=1}(\bW^{UP}_{\ell} + \bm{I}^{UP}) \right) \bX^{UP}_{0}
\end{multlined}
\end{equation}

By assuming enough small $\Delta \bW^{UP}_{\ell}$ as in Assumption~\ref{assume:small_grad_norm} and approximating the difference to first-order terms, we have Eq.~\eqref{eq:first_order}.
\begin{equation}
\label{eq:first_order}
\tilde{\bX}^{UP}_{L} - \bX^{UP}_{L} \approx \sum^{L}_{\ell=1} {\Delta \bW^{UP}_{\ell} \prod^{L}_{k \neq \ell} \p{ \bW^{UP}_{k} + \bm{I} } \bX^{UP}_{0}}
\end{equation}

From the submultiplicative property, the norm of Eq.~\eqref{eq:first_order} can be bounded as Eq.~\eqref{eq:bound_diff}
\begin{equation}
\label{eq:bound_diff}
\norm{ \tilde{\bX}^{UP}_{L} - \bX^{UP}_{L} } \le L \epsilon \p{ 1 + \delta }^{L-1} \norm{ \bX^{UP}_{0} }
\end{equation}
\end{proof}




Next, we analyze the accumulated output shift bound of the width-upscaled residual network with function-preserving lower layers and non-preserving upper layers. A function-preserving layer has frozen $\bW^{12}$ which is zero-initialized and a non-preserving layer has trainable $\bW^{12}$. If the original function is preserved for $L_{fp}$ lower layers, the $L - L_{fp}$ upper layers may contribute to forgetting, while allowing greater learning opportunities. With this intuition and Assumption~\ref{assume:smaller_weight_norm}, we extend Proposition~\ref{pro:width_upscaling} to Proposition~\ref{pro:preserv_width_upscaling}.

\begin{assumption}
\label{assume:smaller_weight_norm}
(Smaller Function-Preserving Weight Norm than Non-Preserving Weight Norm). For function-preserving lower layers $1 \le \ell \le L_{fp}$, the norm of its upscaled weight matrix is bounded by a small constant $\delta_{fp} > 0$, i.e., $\norm{ \bW^{UP}_{\ell} } \le \delta_{fp}$. On the other hand, for non-preserving upper layers $L_{fp} < \ell \le L$, the norm of its upscaled weight matrix is bounded by a small constant $\delta_{np} > 0$, i.e., $\norm{ \bW^{UP}_{\ell} } \le \delta_{np}$. Now, we assume that $\delta_{fp} < \delta_{np}$ since $\bW^{12}$ is frozen as zero-initialized for $\delta_{fp}$ while $\bW^{12}$ is trainable for $\delta_{np}$.
\end{assumption}

\noindent \textbf{Proposition~\ref{pro:preserv_width_upscaling}.} \textit{The accumulated output shift of the width-upscaled residual network with function-preserving lower layers $1 \le \ell \le L_{fp}$ and non-preserving upper layers $L_{fp} < \ell \le L$  is bounded by Eq.~\eqref{eq:bound_preserv_width_upscaling}.}

\begin{equation}
\label{eq:bound_preserv_width_upscaling}
\begin{multlined}
\norm{ \tilde{\bX}^{UP}_{L} - \bX^{UP}_{L} } \\
\le \p{ L - L_{fp} } \epsilon \p{ 1 + \delta_{np} }^{L - 1} \p{ \frac{1 + \delta_{fp}}{1 + \delta_{np}} }^{L_{fp}} \norm{ \bX^{UP}_{0} }
\end{multlined}
\end{equation}
where $\tilde{\bX}^{UP}_{L}$ denotes updated output of width-upscaled residual network as defined by Definition~\ref{def:updated_width_upscaling} and $\delta_{fp}$ and $\delta_{np}$ denotes upper bound of norm of upscaled weight matrix for function-preserving lower layers and non-preserving upper layers, respectively, under Assumption~\ref{assume:smaller_weight_norm}.

\begin{proof}
By taking difference between Eq.~\eqref{eq:acc_output} and Eq.~\eqref{eq:acc_updated_output} for $L_{fp}+1 \le \ell \le L$, we have Eq.~\eqref{eq:acc_preserv_diff}.
\begin{equation}
\label{eq:acc_preserv_diff}
\begin{multlined}
\tilde{\bX}^{UP}_{L} - \bX^{UP}_{L} = \left( \prod^{L}_{\ell=L_{fp}+1}\p{ \bW^{UP}_{\ell} + \Delta\bW^{UP}_{\ell} + \bm{I}^{UP} } \right. \\
\left. - \prod^{L}_{\ell=L_{fp}+1}\p{ \bW^{UP}_{\ell} + \bm{I}^{UP} } \right) \bX^{UP}_{L_{fp}}
\end{multlined}
\end{equation}

By assuming enough small $\Delta \bW^{UP}_{\ell}$ as in Assumption~\ref{assume:small_grad_norm} and approximating the difference to first-order terms, we have Eq.~\eqref{eq:preserv_first_order}.
\begin{equation}
\label{eq:preserv_first_order}
\begin{aligned}
&\tilde{\bX}^{UP}_{L} - \bX^{UP}_{L} \\
&\approx \sum^{L}_{\ell=L_{fp}+1} {\Delta \bW^{UP}_{\ell} \prod^{L}_{k = L_{fp} + 1, k \neq \ell} \p{ \bW^{UP}_{k} + \bm{I} } \bX^{UP}_{L_{fp}}} \\
&\begin{multlined}
= \sum^{L}_{\ell=L_{fp}+1} \Delta \bW^{UP}_{\ell} \prod^{L}_{k = L_{fp} + 1, k \neq \ell} \p{ \bW^{UP}_{k} + \bm{I} } \\
\prod^{L_{fp}}_{\ell=1} \p{ \bW^{UP}_{\ell} + \bm{I} } \bX^{UP}_{0}
\end{multlined}
\end{aligned}
\end{equation}

From the submultiplicative property, the norm of Eq.~\eqref{eq:preserv_first_order} can be bounded as Eq.~\eqref{eq:preserv_bound_diff}
\begin{equation}
\label{eq:preserv_bound_diff}
\begin{multlined}
\norm{ \tilde{\bX}^{UP}_{L} - \bX^{UP}_{L} } \\
\le \p{ L - L_{fp} } \epsilon \p{ 1 + \delta_{np} }^{L - 1} \p{ \frac{1 + \delta_{fp}}{1 + \delta_{np}} }^{L_{fp}} \norm{ \bX^{UP}_{0} }
\end{multlined}
\end{equation}

Note that, since $\delta_{fp} < \delta_{np}$ from Assumption~\ref{assume:smaller_weight_norm}, the forgetting bound increases exponentially with decreasing $L_{fp}$. At its extremes, forgetting happens the most\,(Eq.~\eqref{eq:preserv_bound_diff} degenerates to Eq.~\eqref{eq:bound_diff}) if no layer preserves the original function, i.e., $L_{fp} = 0$, whereas the forgetting bound becomes 0 if every layer preserves the original function, i.e., $L_{fp} = L$.

This completes the proof.
\end{proof}

\noindent \textbf{Corollary~\ref{cor:L_fp}.}
\textit{Forgetting increases exponentially with decreasing $L_{fp}$, number of function preserving $\bW^{12}_{\ell}$ blocks in lower layers $1 \le \ell \le L_{fp}$.}

\begin{proof}
Immediate from Proposition~\ref{pro:preserv_width_upscaling}.
\end{proof}

\section{Convergence Analysis}
\label{sec:convergence_analysis}

In this section, we analyze convergence of standard \emph{Continual Learning} and \emph{Routing-based Continual Learning}. First of all, we define \emph{Continual Learning} and \emph{Routing-based Continual Learning} formally and make the following assumptions for the $i$-th task loss function $F^{i}$, as in many other relevant studies\,\cite{HierFAVG,AdaptiveFL}.

\begin{definition}
\label{def:cl}
(Continual Learning) Given a sequence of tasks $\mathcal{T} = \set{T^{1}, T^{2}, \cdots , T^{N}}$, each task $T^{i}$ involves a set of data examples $\p{\x, y} \in \gD^{i}$ and loss of predictions $F^{i}\p{\w} \triangleq \sum_{\p{\x, y} \in \gD^{i}}{\frac{1}{\abs{\gD^{i}}} \mathcal{L}\p{\w, \x, y}}$ with a loss function $\mathcal{L}$ parameterized by $\w$ over $\gD^{i}$. If all task data could be accessed at once, the ideal continual learning objective would be defined by Eq.~\eqref{eq:cl}.

\begin{equation}
\label{eq:cl}
\w^{*} = \argmin_{\w} \frac{1}{N} \sum^{N}_{i=1} { F^{i}\p{\w} }
\end{equation}
\end{definition}

\begin{definition}
\label{def:route_cl}
(Routing-based Continual Learning) Routing-based continual learning extends Definition~\ref{def:cl} by routing each data sample to the most relevant group weights among multiple candidate groups such that it achieves smaller loss of predictions.
\end{definition}

\begin{assumption}
\label{assume:F^i}
    For every task $i$, (1) $ F^{i} $ is convex;
    (2) $ F^{i} $ is $ \rho $-Lipschitz, i.e., $ \norm{ F^{i}(\w) - F^{i}(\w') } \le \rho\norm{ \w - \w' } $ for any $\w$ and $ \w' $; and
    (3) $ F^{i} $ is $ \beta $-smooth, i.e., $ \norm{ \g F^{i}(\w) - \g F^{i}(\w') } \le \beta\norm{ \w - \w' } $ for any $\w$ and $ \w' $.
\end{assumption}

Next, we analyze the convergence bound of standard \emph{Continual Learning} from the perspective of task weight divergence and extend it to \emph{Routing-based Continual Learning}. We start by expressing task weight update and defining virtual global task and task weight divergence. Starting from the previous task weights $\w^{i-1}, 1 \le i \le N$, $i$-th task weight update can be expressed as Eq.~\eqref{eq:w^i}. $w^{0}_{0}$ is the very initial weights like a randomly initialized one or a pre-trained one.
\begin{equation}
\label{eq:w^i}
\w^{i}_{t} \triangleq \begin{dcases}
    \w^{i-1}_{t} & \text{if} ~ t = 0 \\
    \w^{i}_{t-1} - \eta\g F^{i}\p{\w^{i}_{t-1}} & \text{if} ~ t  > 0
\end{dcases}
\end{equation}

\begin{definition}
(Virtual Global Task)
We define a virtually aggregated global task such that it consists of datasets of all tasks $\gD^{G} \triangleq \cup_{1 \le i \le N}{\gD^{i}}$. Accordingly, global task loss can be expressed as $F^{G}\p{\w^{G}} \triangleq \sum_{\p{\x, y} \in \gD^{G}}{\frac{1}{\abs{\gD^{G}}} \mathcal{L}\p{\w^{G}, \x, y}}$. Starting from the same aforementioned initial weights $w^{0}_{0}$, $\w^{G}$ is updated by Eq.~\eqref{eq:w^G}.
\begin{equation}
\label{eq:w^G}
\w^{G}_{t} \triangleq \begin{dcases}
    \w^{0}_{t} & \text{if} ~ t = 0 \\
    \w^{G}_{t-1} - \eta\g F^{G}\p{\w^{G}_{t-1}} & \text{if} ~ t  > 0
\end{dcases}
\end{equation}
\end{definition}

\noindent \textbf{Theorem~\ref{thm:logit_routed_cl}.} \textit{\emph{Routing-based Continual Learning} achieves lower convergence than standard \emph{Continual Learning}.}

\begin{proof}
Firstly, we model a continual learning error by bounding norm of task weight divergence $\norm{ \w^{i}_{t} - \w^{G}_{t} }$.
\begin{align*}
    &\norm{ \w^{i}_{t} - \w^{G}_{t} } \\
    & \begin{multlined}[\linewidth]
        = \norm{ \w^{i}_{t-1} - \eta\g F^{i}\p{\w^{i}_{t-1}} - \w^{G}_{t-1} + \eta\g F^{G}\p{\w^{G}_{t-1}} } \\
        \text{(from Eq.~\eqref{eq:w^i} and Eq.~\eqref{eq:w^G})}
    \end{multlined} \\
    & \begin{multlined}[\linewidth]
        = \lVert \w^{i}_{t-1} - \w^{G}_{t-1} - \eta \lparen \g F^{G}\p{\w^{i}_{t-1}} - \g F^{G}\p{\w^{G}_{t-1}} \\
        + \g F^{i}\p{\w^{i}_{t-1}} - \g F^{G}\p{\w^{i}_{t-1}} \rparen \rVert \\
        \text{(adding a zero term)}
    \end{multlined} \\
    & \begin{multlined}[\linewidth]
        \leq \norm{ \w^{i}_{t-1} - \w^{G}_{t-1} } + \eta\norm{ \g F^{G}\p{\w^{i}_{t-1}} - \g F^{G}\p{\w^{G}_{t-1}} } \\
        + \eta\norm{ \g F^{i}\p{\w^{i}_{t-1}} - \g F^{G}\p{\w^{i}_{t-1}} } \\
        \text{(from triangle inequality)} 
    \end{multlined} \\
    & \begin{multlined}[\linewidth]
        \leq \p{ \eta\beta + 1 }\norm{ \w^{i}_{t-1} - \w^{G}_{t-1} } \\
        + \eta\norm{ \g F^{i}\p{\w^{i}_{t-1}} - \g F^{G}\p{\w^{i}_{t-1}} } \\
        \text{(from $ \beta $-smoothness of $ F^{G} $)}
    \end{multlined} \\
\end{align*}

To conclude this proof, because this is a recursive inequality, it suffices to show that, in \emph{Routing-based Continual Learning}, the last term $\norm{ \g F^{i}\p{\w} - \g F^{G}\p{\w} }$ becomes smaller for any weights $\w$.

We assume that the negative log likelihood loss is used for any $ j $-th data example and any weights $\w$, which can be expressed as Eq.~\eqref{eq:loss_j}.
\begin{equation}
\label{eq:loss_j}
    \mathcal{L}\p{ \w, \x_{j}, y_{j} } = -\sum^{C}_{c=1}{ \gI_{jc} \log{ s_{jc}\p{\w} } },
\end{equation}
where $ \gI_{jc} \triangleq \gI\p{ y_j = c } $ and $ s_{jc}\p{\w} \triangleq s\p{ \w, \x_{j}, y_{j} = c } $ denote the ground truth indicator and the softmax value, respectively, on the $ j $-th data example belonging to the class $ c $. Then, we can reformulate the task loss as Eq.~\eqref{eq:loss_F^i}.
\begin{equation}
\begin{split}
\label{eq:loss_F^i}
    F^{i}\p{\w} & = -\sum_{ j \in \gD^{i} }{ \frac{ 1 }{ \abs{\gD^{i}} } \sum^{C}_{c=1}{ \gI_{jc} \log{ s_{jc}\p{\w} } } } \\
    & = -\sum^{C}_{c=1}{ \frac{ \abs{\gD^{ic}} }{ \abs{\gD^{i}} } \sum_{ j \in \gD^{ic} }{ \frac{ 1 }{ \abs{\gD^{ic}} } \log{ s_{jc}\p{\w} } } } \\
    & = \sum^{C}_{c=1}{ \gP\p{ y_{j} = c | j \in \gD^{i} } \gE_{ \gD^{ic} }\bkt{ -\log{ s_{jc}\p{\w} } } } \\
    & = \sum^{C}_{c=1}{ \gP_{ic} \gE_{ \gD^{ic} } } \quad \text{(for brevity)}
\end{split}
\end{equation}
where $ \gE_{ \gD^{ic} } $ denotes the expectation over $ \gD^{ic} $, which is the set of data examples belonging to the class $ c $ of the $ i $-th task and can be defined as $ \gD^{ic} \triangleq \set{ j \in \gD^{i} | y_j = c } $. Based on the similar deduction, global task loss $ F^{G}\p{\w}$ can be reformulated as Eq.~\eqref{eq:loss_F^G}.
\begin{equation}
\begin{split}
\label{eq:loss_F^G}
    F^{G}\p{\w} &= \sum^{C}_{c=1}{ \gP\p{ y_{j} = c | j \in \gD^{G} } \gE_{ \gD^{Gc} }\bkt{ -\log{ s_{jc}\p{\w} } } } \\
    & = \sum^{C}_{c=1}{ \gP_{Gc} \gE_{ \gD^{Gc} } } \quad \text{(for brevity)}
\end{split}
\end{equation}
where $ \gD^{Gc} \triangleq \set{ j \in \gD^{G} | y_j = c } $.

Then, we obtain Eq.~\eqref{eq:grad_div}
\begin{equation}
\begin{aligned}
\label{eq:grad_div}
    & \norm{ \g F^{i}\p{\w} - \g F^{G}\p{\w} } \\
    & \begin{multlined}[\linewidth]
        = \norm{ \sum^{C}_{c=1} { \gP_{ic} \g \gE_{ \gD^{ic} } - \gP_{Gc} \g \gE_{ \gD^{Gc} } } } \\
        \normal{(from Eq.~\eqref{eq:loss_F^i} and \eqref{eq:loss_F^G} and the linearity of gradient)}
    \end{multlined} \\
    & \begin{multlined}[\linewidth]
        = \norm{ \sum^{C}_{c=1} { \gP_{ic} \g \gE_{ \gD^{ic} } - \gP_{Gc} \g \gE_{ \gD^{ic} } + \gP_{Gc} \g \gE_{ \gD^{ic} } - \gP_{Gc} \g \gE_{ \gD^{Gc} } } } \\
        \normal{(adding a zero term)}
    \end{multlined} \\
    & \begin{multlined}[\linewidth]
        \le \sum^{C}_{c=1} \abs{ \gP_{ic} - \gP_{Gc} } \norm{ \g \gE_{ \gD^{ic} } } + \sum^{C}_{c=1} \abs{ \gP_{Gc} } \norm{ \g \gE_{ \gD^{ic} } - \g \gE_{ \gD^{Gc} } } \\
        \normal{(from triangle inequality)}
    \end{multlined}
\end{aligned}
\end{equation}

From the definition of \emph{Routing-based Continual Learning}
in Definition~\ref{def:route_cl}, each task dataset is further
partitioned into multiple routed groups $g$.
Let $\gP_{igc}$ and $\gP_{Ggc}$ denote the joint probability that
a data sample belongs to group $g$ and class $c$ for task $i$
and the global task $G$, respectively. Then, for each class $c$,
we can decompose the class probabilities as
$\gP_{ic} = \sum_g \gP_{igc}$ and $\gP_{Gc} = \sum_g \gP_{Ggc}$, and hence
the data distribution difference in Eq.~\eqref{eq:grad_div} becomes
\begin{equation}
\label{eq:grouped_dists}
\bigl| \gP_{ic} - \gP_{Gc} \bigr|
  = \Bigl| \sum_g \bigl( \gP_{igc} - \gP_{Ggc} \bigr) \Bigr|.
\end{equation}
Assuming that the router is well aligned with the global
distribution so that, for every group $g$ and class $c$,
\[
\bigl| \gP_{igc} - \gP_{Ggc} \bigr|
   \le \bigl| \gP_{ic} - \gP_{Gc} \bigr|,
\]
the resulting bound on $\norm{ \g F^{i}\p{\w} - \g F^{G}\p{\w} }$ for
\emph{Routing-based Continual Learning} is tighter than that of
standard \emph{Continual Learning}, and thus the recursive inequality
yields a smaller task weight divergence $\norm{ \w^{i}_{t} - \w^{G}_{t} }$.

This completes the proof.
\end{proof}